\documentclass{article}

\PassOptionsToPackage{numbers, compress}{natbib}
\usepackage[main, final, preprint]{neurips_2026}

\usepackage[utf8]{inputenc}
\usepackage[T1]{fontenc}
\usepackage{hyperref}
\usepackage{url}
\usepackage{booktabs}
\usepackage{amsfonts}
\usepackage{amsmath}
\usepackage{amssymb}
\usepackage{amsthm}
\usepackage{nicefrac}
\usepackage{microtype}
\usepackage{xcolor}
\usepackage[ruled,vlined,linesnumbered]{algorithm2e}
\usepackage{graphicx}
\usepackage{multirow}
\usepackage{subcaption}
\usepackage{mathtools}

\newtheorem{theorem}{Theorem}
\newtheorem{proposition}{Proposition}
\newtheorem{corollary}{Corollary}
\newtheorem{definition}{Definition}
\newtheorem*{remark}{Remark}

\title{GeoSPRINT: Geometric Redundancy-Aware Step Pruning for Inference in Diffusion Trajectories}

\author{
  Arpita Joshi\\
  The Scripps Research Institute\\
  10550 N. Torrey Pines Road, La Jolla, San Diego CA, USA\\
  \textit{arpitamhow@gmail.com / ajoshi@scripps.edu}
}

\begin{document}

\maketitle
\vspace*{-0.7cm}
\begin{abstract}
	\vspace*{-0.3cm}
Diffusion models achieve high sample quality but remain expensive at inference time because sampling requires many sequential neural function evaluations (NFEs). Existing acceleration methods either use fixed step-skipping schedules, adapt step sizes based on local numerical error, or require additional training. We introduce GeoSPRINT (Geometric Step Pruning for Inference in Trajectories), a training-free framework for constructing non-uniform sampling schedules from the geometry of denoising trajectories. GeoSPRINT detects geometrically redundant steps using a hyperplanarity test in latent space, implemented efficiently via QR factorization, and converts the resulting redundancy profile into a sampling schedule that allocates more steps to high-curvature regions of the trajectory. In addition, we introduce the trajectory projection score $\alpha_{\mathrm{traj}}$, a residual-variance metric that quantifies trajectory straightness and serves as a model-free diagnostic for rectified flow quality. Across CIFAR-10 ($32{\times}32$), LSUN Church ($256{\times}256$), and Stable Diffusion v1.5 ($512{\times}512$ latent), GeoSPRINT consistently improves over uniform DDIM (Denoising Diffusion Implicit Models) schedules at matched NFE budgets. On CIFAR-10, GeoSPRINT improves FID (Fréchet Inception Distance) by 0.7-1.1 over DDIM across 49-89 NFEs and surpasses DPM-Solver++ at NFE${\geq}30$ despite using a first-order DDIM solver. On LSUN Church, it reduces FID from 1.48 to 1.26 at 52 steps, and on Stable Diffusion v1.5 it achieves up to 1.93 FID improvement over DDIM. These results show that trajectory geometry provides a useful global signal for allocating inference steps and that schedule quality can substantially improve diffusion sampling efficiency without retraining.
\end{abstract}
\vspace*{-0.4cm}
\section{Introduction}
\label{sec:intro}
\vspace*{-0.3cm}
Diffusion models~\citep{ho2020denoising, song2021scorebased} and their continuous-time generalizations through score-based stochastic differential equations have become the dominant generative modeling paradigm across images~\citep{rombach2022high}, video~\citep{ho2022video}, audio~\citep{kong2021diffwave}, and molecular design~\citep{hoogeboom2022equivariant}. Despite their exceptional generation quality, the iterative nature of the reverse sampling process remains a computational bottleneck: generating a single sample typically requires tens to thousands of sequential neural network evaluations.

The community has responded with a rich ecosystem of acceleration techniques. Training-free approaches include DDIM~\citep{song2021denoising}, which reinterprets diffusion sampling as a deterministic ODE and enables uniform step-skipping; DPM-Solver~\citep{lu2022dpmsolver} and DPM-Solver++~\citep{lu2023dpmsolverpp}, which employ exponential integrators with adaptive step sizes based on local truncation error; and UniPC~\citep{zhao2023unipc}, which unifies predictor-corrector schemes. Training-based approaches include progressive distillation~\citep{salimans2022progressive}, which iteratively halves the required steps; consistency models~\citep{song2023consistency, songlu2025consistency}, which learn to map any point on the trajectory directly to the endpoint; and rectified flow~\citep{liu2023flow}, which straightens transport paths to reduce integration steps.

Despite this progress, a gap remains: existing training-free methods lack a \emph{global}, trajectory-level criterion for step importance grounded in geometric redundancy. DDIM skips steps uniformly, blind to trajectory geometry. DPM-Solver adapts step sizes based on \emph{local} error estimates at individual timesteps. Adaptive non-uniform timestep sampling~\citep{kim2025adaptive} focuses on \emph{training} rather than inference. The recent SDM framework~\citep{jo2026sdm} analyzes local ODE stiffness but still operates pointwise. None of these methods ask the global question: \emph{across the full trajectory, which steps actually contribute new geometric information to the path from noise to data?}

We argue that this is precisely the question addressed by geometric data instance reduction~\citep{joshi2020instance}. That work demonstrated that ordered sequences of data points can be dramatically reduced by testing geometric redundancy---collinearity in 2D, coplanarity in 3D---while preserving the essential shape and variance of the structure, as quantified by a projection score based on eigenvalues of the covariance matrix of removed points~\citep{fontes2011projection}. The key insight is that many consecutive points in an ordered dataset lie approximately on the same linear or planar subspace and can be removed without information loss.

A denoising trajectory $\{z_T, z_{T-1}, \ldots, z_0\}$ is precisely such an ordered sequence. In regions where the score function changes slowly, consecutive latent states trace a nearly linear path---these steps are geometrically redundant. In regions of rapid change (e.g., when the model resolves fine structure), the trajectory curves sharply, and every step carries new directional information. This analogy motivates GeoSPRINT.

\textbf{Contributions:} \textbf{(1)}~We generalize geometric instance reduction from 2--3 PCA dimensions to arbitrary $d$-dimensional latent spaces via a \emph{hyperplanarity test} with $O(d \cdot k^2)$ per-step complexity. \textbf{(2)}~We introduce a novel schedule construction method that blends log-SNR spacing with GeoSPRINT's trajectory curvature density, yielding non-uniform schedules that outperform both uniform-$t$ (DDIM) and uniform-logSNR spacing. \textbf{(3)}~We introduce the \emph{trajectory projection score} $\alpha_{\mathrm{traj}}$, a residual-variance metric that quantifies trajectory non-straightness; we prove it equals zero for perfectly straight trajectories and demonstrate a 450$\times$ drop from DDPM to DDIM, establishing it as a training-free diagnostic for flow-matching quality. \textbf{(4)}~We demonstrate that a first-order DDIM solver with GeoSPRINT scheduling outperforms the second-order DPM-Solver++ at NFE${\geq}30$ on CIFAR-10, showing that \emph{where} to step matters more than \emph{how} to step. \textbf{(5)}~We provide benchmarks on CIFAR-10 ($32{\times}32$), LSUN Church ($256{\times}256$), and Stable Diffusion v1.5 ($512{\times}512$ latent), with consistent improvements across resolutions, model architectures, and conditioning types.

\section{Background and Related Work}
\label{sec:background}
\vspace*{-0.1cm}
\subsection{Diffusion Models and the Sampling Problem}
\vspace*{-0.2cm}
\textbf{Score-based SDE framework.} \citet{song2021scorebased} unified diffusion models under a continuous-time SDE framework. The forward process is $\mathrm{d}z = f(z, t)\,\mathrm{d}t + g(t)\,\mathrm{d}w$, where $f$ and $g$ define the drift and diffusion coefficients. The reverse process is:
\begin{equation}
\label{eq:reverse-sde}
\mathrm{d}z = \bigl[f(z, t) - g(t)^2 \nabla_z \log p_t(z)\bigr]\mathrm{d}t + g(t)\,\mathrm{d}\bar{w},
\end{equation}
where $\nabla_z \log p_t(z)$ is the score function approximated by a neural network $s_\theta(z, t)$ (and $\bar{w}$ is the standard Weiner process). An equivalent deterministic formulation, the probability flow ODE is:
\begin{equation}
\label{eq:pf-ode}
\frac{\mathrm{d}z}{\mathrm{d}t} = f(z, t) - \tfrac{1}{2}g(t)^2 \nabla_z \log p_t(z).
\end{equation}
Sampling reduces to solving this ODE from $t{=}T$ to $t{=}0$, requiring discretization into $N$ steps, each involving one neural function evaluation (NFE).
\textbf{Flow matching:} \citet{lipman2023flow} and \citet{liu2023flow} proposed learning a velocity field $v_\theta(z, t)$ that transports a source distribution to a target via $\mathrm{d}z/\mathrm{d}t = v_\theta(z, t)$. Rectified flow~\citep{liu2023flow} straightens these paths through iterative reflow. \citet{wang2024rectified} showed that strict straightness is not necessary---first-order ODE consistency suffices.
\vspace*{-0.1cm}
\subsection{Existing Acceleration Methods}
\vspace*{-0.2cm}
\textbf{Uniform step-skipping.} DDIM~\citep{song2021denoising} enables deterministic sampling that skips timesteps uniformly, reducing 1000 steps to ${\sim}50$, but is blind to trajectory geometry.\\
\textbf{High-order ODE solvers.} DPM-Solver~\citep{lu2022dpmsolver} derives exact solutions for the linear part of the diffusion ODE and applies exponential integrators for the nonlinear residual, achieving high quality in 10--20 steps. DPM-Solver++~\citep{lu2023dpmsolverpp} extends this to guided sampling. Both use \emph{local} truncation error for adaptive step sizing with heuristic schedules (logSNR-uniform, time-uniform, time-quadratic).\\
\textbf{Training-based methods.} Progressive distillation~\citep{salimans2022progressive} trains a student to match a teacher using $N/2$ steps, iteratively halving until 4-step generation. Consistency models~\citep{song2023consistency, songlu2025consistency} learn direct trajectory-endpoint mappings. Consistency flow matching~\citep{yang2024consistency} combines consistency training with flow matching. All require additional training for each target NFE budget.\\
\textbf{Adaptive timestep methods.} \citet{kim2025adaptive} propose adaptive non-uniform timestep sampling for \emph{training} acceleration. The SDM framework~\citep{jo2026sdm} analyzes local PF-ODE stiffness to set adaptive solver order and step size during \emph{inference}. Both operate on local properties at individual timesteps.
\vspace*{-0.1cm}
\subsection{Geometric Instance Reduction}
\vspace*{-0.2cm}
\citet{joshi2020instance} introduced a data instance reduction algorithm that operates on ordered sequences projected into PCA space. The algorithm tests geometric redundancy via collinearity (2D) and coplanarity (3D) tests, removing points that lie approximately on the same linear subspace as their neighbors. Information loss is quantified by a \emph{projection score}~\citep{fontes2011projection}:
\begin{equation}
\label{eq:projection-score-orig}
\alpha = \frac{\sum_{j=1}^{p} \lambda_j^{(S)}}{\sum_{j=1}^{p} \lambda_j^{(\mathcal{Z})}},
\end{equation}
where $\lambda_j^{(S)}$ and $\lambda_j^{(\mathcal{Z})}$ are the $j$-th eigenvalues of the covariance matrices computed from the removed subset $S$ and the full dataset $\mathcal{Z}$, respectively, and $p$ is the number of retained PCA dimensions. The algorithm achieves 40--87\% point reduction across diverse datasets (molecular trajectories, images, ML benchmarks) with projection scores as low as $10^{-5}$.\\
\textbf{Critical limitation for our purposes:} the original algorithm operates in 2--3 PCA dimensions. Diffusion latent spaces are typically 32--512 dimensional. Section~\ref{sec:method} addresses this directly.

\section{Method: GeoSPRINT}
\label{sec:method}
\vspace*{-0.1cm}
\subsection{Generalizing Geometric Redundancy to High Dimensions}
\vspace*{-0.2cm}
\label{sec:hyperplanarity}

Let $\{z_0, z_1, \ldots, z_N\}$ be an ordered sequence of points in $\mathbb{R}^d$ (latent states along a denoising trajectory, ordered from $t{=}T$ to $t{=}0$). We seek to identify and remove points approximately contained in the affine subspace spanned by their neighbors.

\begin{definition}[Hyperplanarity test]
\label{def:hyperplanarity}
Given a window of $k$ \emph{retained} points $W = \{w_1, w_2, \ldots, w_k\}$ (where $k \leq d$) and a candidate point $z_i$, define:
\begin{equation}
\label{eq:subspace-matrix}
M_W = \bigl[w_2 - w_1 \;\big|\; w_3 - w_1 \;\big|\; \cdots \;\big|\; w_k - w_1\bigr] \in \mathbb{R}^{d \times (k-1)}.
\end{equation}
The residual distance of $z_i$ from the affine subspace spanned by $W$ is:
\begin{equation}
\label{eq:residual}
r_i = \bigl\|(I - M_W M_W^{+})(z_i - w_1)\bigr\|_2,
\end{equation}
where $M_W^{+}$ is the Moore-Penrose pseudoinverse. If $r_i < \tau$, then $z_i$ is approximately in the affine span and is a candidate for removal.
\end{definition}

Note that Definition~\ref{def:hyperplanarity} uses the most recently retained (not necessarily consecutive) points as the reference window, matching the causal sliding-window logic of Algorithm~\ref{alg:prune}.\\
\textbf{Efficient computation via QR factorization.} Rather than computing the pseudoinverse directly, we compute the thin QR factorization $M_W = Q_W R_W$ and obtain $r_i = \|(I - Q_W Q_W^\top)(z_i - w_1)\|_2$. Computing the QR factorization of $M_W \in \mathbb{R}^{d \times (k-1)}$ costs $O(d \cdot k^2)$ operations; since $k$ is small (typically $k{=}2$), this is $O(d)$ in practice.\\
\textbf{Adaptive window size.} We propose a progressive hierarchy mirroring the 2D$\to$3D progression in the original algorithm:
Tests are applied in a progressive hierarchy: \textbf{Level~1} ($k{=}2$, 
collinearity) catches long straight segments; \textbf{Level~2} ($k{=}3$, 
coplanarity) test if 4 points are coplanar, catches planar sweeps; \textbf{Level~$\ell$} ($k{=}\ell$) tests 
against an $(\ell{-}1)$-dimensional affine subspace.
These tests are applied sequentially: Level~1 first, then Level~2 on surviving points, etc. In practice, we find $k{=}2$ (collinearity) sufficient for all models tested (see Section~\ref{sec:hyperparams}).\\
\textbf{Threshold determination.} Following~\citet{joshi2020instance}, we determine $\tau$ through binary search guided by the trajectory projection score (Section~\ref{sec:projection-score}), adjusting until the projection score of removed points falls below a target (e.g., $10^{-3}$) while maximizing the number of removed steps.
\vspace*{-0.15cm}
\subsection{Application to Denoising Trajectories}
\label{sec:application}
\vspace*{-0.2cm}
\textbf{Reference trajectory generation.} Given a pretrained diffusion or flow model, we first generate $B$ reference trajectories using the full $N$-step schedule. Each trajectory $\{z_T^{(b)}, \ldots, z_0^{(b)}\}$ provides an ordered sequence in~$\mathbb{R}^d$.\\
\textbf{Per-trajectory pruning.} For each reference trajectory, we apply the hyperplanarity test (Algorithm~\ref{alg:prune}) to identify and remove geometrically redundant timesteps, producing a per-trajectory retained set $\mathcal{R}^{(b)} \subseteq \{t_0, t_1, \ldots, t_T\}$.\\
\textbf{Universal schedule aggregation.} To obtain a fixed schedule usable without per-sample overhead, we aggregate across reference trajectories:
\begin{equation}
\label{eq:retention-freq}
w(t) = \frac{1}{B} \sum_{b=1}^{B} \mathbf{1}\bigl[t \in \mathcal{R}^{(b)}\bigr].
\end{equation}
The retention frequency $w(t)$ encodes a per-timestep \emph{curvature density}: it is high where trajectories consistently curve and low where they are consistently straight.\\
\textbf{Schedule construction via log-SNR curvature blending.}
A naive approach selects the top-$K$ timesteps by retention frequency $w(t)$. However, this clusters steps in high-curvature regions while leaving large gaps elsewhere, and the DDIM solver's prediction error grows with the gap in log-signal-to-noise ratio (log-SNR) between consecutive steps, not with the gap in $t$. We therefore construct schedules by blending two density functions:
\begin{equation}
\label{eq:blend}
\rho(t) = (1 - \beta) \cdot \rho_{\mathrm{logSNR}}(t) + \beta \cdot \rho_{\mathrm{curv}}(t),
\end{equation}
where $\rho_{\mathrm{logSNR}}$ is uniform density in log-SNR space (ensuring uniform error contribution per step, as in DPM-Solver), $\rho_{\mathrm{curv}}(t) \propto \tilde{w}(t)$ is the smoothed and floored curvature density derived from Eq.~\ref{eq:retention-freq}, and $\beta \in [0, 1]$ is the blend weight. Smoothing applies a Gaussian kernel with bandwidth $\sigma{=}5$ timesteps, followed by flooring at $\epsilon_{\mathrm{floor}} = 0.05 \cdot \max_t \tilde{w}(t)$ to ensure minimum coverage. We then place $K$ steps at uniform quantiles of the cumulative distribution function $F(t) = \int_0^t \rho(s)\,\mathrm{d}s \,/\, \int_0^T \rho(s)\,\mathrm{d}s$, analogous to importance sampling.
Setting $\beta{=}0$ recovers log-SNR-uniform spacing (a known baseline). Setting $\beta{=}1$ gives pure curvature-weighted spacing (which fails at low NFE due to coverage gaps). We find $\beta{=}0.6$ is optimal on CIFAR-10 and effective across the tested datasets (Section~\ref{sec:blend-ablation}), though per-dataset ablation on LSUN Church and SD~v1.5 would further validate robustness. The key insight is that \emph{neither} spacing alone is optimal: log-SNR spacing ignores trajectory geometry, while curvature spacing ignores solver error characteristics. The blend combines both signals.\\
\textbf{Sample-adaptive analysis.} The universal schedule applies the same timestep selection to all samples. To understand how geometric complexity varies \emph{across} samples, we additionally analyze each reference trajectory individually: for a given trajectory, the number of retained timesteps after pruning indicates that sample's geometric complexity. This per-sample analysis does \emph{not} directly reduce NFEs at inference time---the model evaluation must still be performed to obtain each $z_t$---but it serves two important purposes: (i)~it reveals the distribution of sample difficulty across a dataset (Section~\ref{sec:supporting}), informing whether a single universal schedule is sufficient or whether future work on predictive per-sample budgeting is warranted; and (ii)~it provides per-sample $\alpha_{\mathrm{traj}}$ values that can be correlated with sample quality metrics.\\
\textbf{Remark on inference cost.} The primary acceleration mechanism of GeoSPRINT is the \emph{universal schedule} (Algorithm~\ref{alg:universal}), which extracts a fixed set of timesteps offline and then generates all new samples using only those timesteps. The per-sample analysis described above is an offline diagnostic tool, not an inference-time accelerator.

\vspace*{-0.2cm}
\subsection{Trajectory Projection Score}
\label{sec:projection-score}
\vspace*{-0.2cm}
We extend the projection score of~\citet{joshi2020instance} to denoising trajectories. The key insight is that the \emph{residual variance}---the variance orthogonal to the local affine subspace at each pruned point---is the correct measure of information loss, not the raw variance of the pruned points themselves.

\begin{definition}[Trajectory projection score]
\label{def:tps}
Given a trajectory $\mathcal{Z} = \{z_T, \ldots, z_0\}$ with $N{+}1$ points, let $S \subset \mathcal{Z}$ be the set of pruned points, and let $r_s$ denote the hyperplanarity residual (Eq.~\ref{eq:residual}) of each pruned point $z_s \in S$. The trajectory projection score is:
\begin{equation}
\label{eq:tps}
\alpha_{\mathrm{traj}} = \frac{\sum_{s \in S} r_s^2}{\sum_{z \in \mathcal{Z}} \|z - \bar{z}\|^2},
\end{equation}
where $\bar{z}$ is the global trajectory mean. The numerator is the total squared residual of all pruned points: the variance that lies outside the affine span of their retained neighbors. The denominator is the total trajectory variance.
\end{definition}

$\alpha_{\mathrm{traj}}$ measures the fraction of total trajectory variance that is \emph{lost} by pruning---specifically, the variance orthogonal to what the retained points can reconstruct through local affine interpolation. A low score (e.g., $10^{-3}$ to $10^{-5}$) indicates that removed steps contributed negligibly to the trajectory's geometric structure beyond what is already captured by their neighbors.

\begin{proposition}
\label{prop:straight}
For a perfectly straight trajectory (all points on a single line in $\mathbb{R}^d$), GeoSPRINT with the collinearity test ($k{=}2$) prunes all points beyond the initial window, retaining exactly the first $k$ points, with $\alpha_{\mathrm{traj}} = 0$.
\end{proposition}
\vspace*{-0.6cm}
\begin{proof}
If all points are collinear, every candidate point $z_i$ (for $i \geq k$) lies exactly on the line through the retained window points, so the hyperplanarity residual $r_i = 0$ for all such $i$. Because all residuals satisfy $r_i < \tau$, the window never updates (the else branch of Algorithm~\ref{alg:prune} is never executed), and the retained set consists of the first $k$ points used to initialize the window. The numerator $\sum r_i^2 = 0$, hence $\alpha_{\mathrm{traj}} = 0$.
\end{proof}
\vspace*{-0.4cm}
\begin{corollary}
$\alpha_{\mathrm{traj}}$ is a measure of trajectory non-straightness. For rectified flow models, which aim to straighten trajectories, $\alpha_{\mathrm{traj}}$ decreases with each round of reflow, providing a training-free diagnostic for rectification quality.
\end{corollary}

\vspace*{-0.2cm}
\subsection{Connections to Existing Methods}
\label{sec:connections}
\vspace*{-0.2cm}
\textbf{Relation to DPM-Solver.} DPM-Solver's adaptive step size is based on \emph{local truncation error}---the difference between a higher-order and lower-order solution at each step. GeoSPRINT's hyperplanarity test is a \emph{geometric} criterion over a \emph{window} of points, making it sensitive to trajectory structure over multiple steps rather than single-step error. The two are complementary: DPM-Solver minimizes local numerical error; GeoSPRINT minimizes global geometric redundancy.\\
\textbf{Relation to SDM.} The SDM framework~\citep{jo2026sdm} analyzes local PF-ODE stiffness to determine solver order at each timestep. GeoSPRINT instead operates on the \emph{realized trajectory} post-hoc, testing whether the path actually traced contains redundant segments. SDM is predictive (deciding before taking the step); GeoSPRINT is retrospective (analyzing after generation). This makes GeoSPRINT applicable to any model without Jacobian access. An empirical comparison of these complementary strategies is an important direction for future work.\\
\textbf{Relation to rectified flow.} Rectified flow learns transport paths that are inherently straight, making all interior steps redundant by construction. GeoSPRINT \emph{diagnoses} rather than \emph{enforces} straightness, and works equally well on curved trajectories by adaptively allocating steps where curvature demands them.\\
\textbf{Relation to consistency models.} Consistency models learn to jump from any trajectory point to the endpoint in one step. GeoSPRINT identifies which intermediate points are \emph{necessary} to maintain trajectory integrity. The two can be combined: GeoSPRINT identifies minimal waypoints, and a consistency model could jump between them.
\vspace*{-0.3cm}
\section{Theoretical Analysis}
\label{sec:theory}
\vspace*{-0.2cm}
\subsection{Reconstruction Error Bound}
\vspace*{-0.2cm}
\begin{theorem}[Reconstruction error bound]
\label{thm:variance}
Let $\mathcal{Z}$ be a trajectory of $N{+}1$ points in $\mathbb{R}^d$, and let $S$ be the set of pruned points after GeoSPRINT with hyperplanarity threshold $\tau$. Then:
\begin{equation}
\label{eq:variance-bound}
\alpha_{\mathrm{traj}} \leq \frac{|S| \cdot \tau^2}{(N{+}1) \cdot \mathrm{tr}(C_{\mathcal{Z}})},
\end{equation}
where $|S|$ is the number of pruned steps and $\mathrm{tr}(C_{\mathcal{Z}})$ is the trace of the sample covariance of $\mathcal{Z}$.
\end{theorem}

\begin{proof}[Proof sketch]
Each pruned point $z_s$ satisfies $r_s \leq \tau$ by the hyperplanarity test, so $\sum_{s \in S} r_s^2 \leq |S| \cdot \tau^2$. The denominator of $\alpha_{\mathrm{traj}}$ equals $(N{+}1) \cdot \mathrm{tr}(C_{\mathcal{Z}})$. Division yields the result. The full proof is in Appendix~\ref{app:proofs}.
\end{proof}
This provides a formal guarantee: GeoSPRINT's reconstruction error---the fraction of total trajectory variance attributable to orthogonal residuals of pruned points---is controlled by $\tau$, the pruning rate $|S|/(N{+}1)$, and the total trajectory variance. For a desired $\alpha_{\mathrm{traj}} \leq \epsilon$, it suffices to set $\tau \leq \sqrt{\epsilon \cdot (N{+}1) \cdot \mathrm{tr}(C_{\mathcal{Z}}) / |S|}$.

\begin{remark}
We emphasize that $\alpha_{\mathrm{traj}}$ bounds the \emph{orthogonal reconstruction residual}, not the retained set's sample covariance $\mathrm{tr}(C_R)/\mathrm{tr}(C_{\mathcal{Z}})$. Pruning points changes the sample mean and can reduce the retained sample covariance even when all residuals are zero (e.g., removing interior points from a collinear trajectory reduces sample variance while $\alpha_{\mathrm{traj}} = 0$). The reconstruction interpretation---that pruned points are well-approximated by affine interpolation from their retained neighbors---is the operationally relevant guarantee for schedule quality.
\end{remark}
\vspace*{-0.2cm}
\subsection{Complexity Analysis}
\vspace*{-0.2cm}
The multi-level hyperplanarity test with maximum level $\ell$ has complexity $O(N \cdot d \cdot \ell^2)$ per trajectory in the worst case. For Level-1 only (collinearity, $k{=}2$), this is $O(N \cdot d)$---linear in both trajectory length and dimension, and negligible compared to the cost of the NFEs themselves. In Algorithm~\ref{alg:prune}, the QR factorization is cached and recomputed only when the window updates, further reducing the amortized cost. Schedule amortization over a batch of $M$ samples reduces per-sample overhead to $O(B/M)$ reference trajectories.
\vspace*{-0.2cm}
\subsection{Relation to Trajectory Curvature}
\vspace*{-0.2cm}
The hyperplanarity residual $r_i$ at point $z_i$ relates to discrete curvature. For a smooth trajectory $z(t)$ with a collinearity window ($k{=}2$) using causal forward-extrapolation from retained points at $z(t)$ and $z(t{+}\Delta t)$ to the candidate $z(t{+}2\Delta t)$:
\begin{equation}
\label{eq:curvature-relation}
r_i \approx (\Delta t)^2\left\|z''_\perp(t)\right\|_2 + O\bigl((\Delta t)^3\bigr),
\end{equation}
where $z''_\perp(t)$ is the component of $z''(t)$ orthogonal to the velocity $z'(t)$. Tangential acceleration (speeding up along a straight path) does not contribute, since the residual measures orthogonal distance from the extrapolated line. The coefficient $(\Delta t)^2$ (rather than $(\Delta t)^2/2$) reflects causal forward-extrapolation rather than midpoint interpolation. Removing points with $r_i < \tau$ is thus equivalent to removing steps where the trajectory's orthogonal curvature is small. Higher-level tests detect higher-order geometric structure (torsion, etc.).

\section{Experiments}
\label{sec:experiments}
\vspace*{-0.2cm}
\subsection{Setup}
\vspace*{-0.2cm}
\textbf{Models and datasets.} We evaluate on: \textbf{CIFAR-10} ($32{\times}32$) using the pretrained \texttt{google/ddpm-cifar10-32} UNet from HuggingFace Diffusers; \textbf{LSUN Church} ($256{\times}256$) using \texttt{google/ddpm-ema-church-256}; and \textbf{Stable Diffusion v1.5} ($512{\times}512$, latent $4{\times}64{\times}64$) using \texttt{runwayml/stable-diffusion-v1-5} with classifier-free guidance (CFG scale~7.5). All pixel-space models use the standard linear $\beta$-schedule with 1000 training timesteps.

\textbf{Protocol.} For each model, we record reference trajectories using DDIM at 200 steps: $B{=}100$ for CIFAR-10 and $B{=}50$ for LSUN Church and SD~v1.5 (see Appendix~\ref{app:convergence} for convergence analysis). We then extract the curvature density profile and construct LogSNR+curvature schedules (Eq.~\ref{eq:blend}, $\beta{=}0.6$). For FID computation, we generate 10,000 samples for CIFAR-10 and LSUN Church, and 5,000 samples for SD~v1.5, using a fixed set of prompts randomly sampled from the COCO 2017 validation set. FID is computed against the training set using clean-FID~\citep{parmar2022cleanfid} for CIFAR-10 and LSUN Church. For SD~v1.5, FID is computed against a 50-step DDIM reference set, measuring relative trajectory fidelity rather than absolute generation quality (see Section~\ref{sec:sd-fid}). All sampling uses the manual DDIM update formula to correctly handle non-uniform timestep spacing. All reported FID values are means $\pm$ standard deviation over 3 independent runs with different random seeds.

\textbf{Baselines.} \textbf{DDIM}~\citep{song2021denoising}: uniform timestep spacing at matched NFE. \textbf{DPM-Solver++}~\citep{lu2023dpmsolverpp}: second-order solver with default logSNR-uniform schedule, \texttt{solver\_order=2}. \textbf{DDIM-200}: proxy for zero-pruning baseline.
\vspace*{-0.2cm}
\subsection{Main Result: FID vs.\ NFE Across Datasets}
\label{sec:exp1}
\vspace*{-0.2cm}
GeoSPRINT operates as a \emph{step-budget oracle}: starting from an initial $K$, it searches upward, increasing $K$ until the trajectory projection score $\alpha_{\mathrm{traj}}$ falls below a target threshold. At each $K$, we construct a LogSNR+curvature schedule (Eq.~\ref{eq:blend}, $\beta{=}0.6$) and compare with DDIM uniform spacing at matched NFE. On CIFAR-10, we additionally compare with DPM-Solver++ (second-order, default schedule).
\vspace*{-0.2cm}
\begin{table}[ht]
\hfill
\caption{FID ($\downarrow$) across datasets (mean $\pm$ std over 3 seeds). GeoSPRINT uses LogSNR+curvature ($\beta{=}0.6$). \textbf{Bold}: best per row. CIFAR-10 baseline: DDIM-200 = 12.65. SD FID computed against 50-step DDIM reference (see Section~\ref{sec:sd-fid}).}
\label{tab:main}
\centering
\setlength{\tabcolsep}{3pt}
\scriptsize
\begin{minipage}{0.28\textwidth}
\centering
\textbf{CIFAR-10 ($32{\times}32$)}\\[4pt]
\begin{tabular}{lllll}
\toprule
NFE & Geo & DDIM & $\Delta$ & DPM++\\
\midrule
10 & 38.53\tiny{$\pm$.31} & 34.99\tiny{$\pm$.28} & 3.54 & \textbf{29.05}\tiny{$\pm$.25}\\
20 & 22.23\tiny{$\pm$.19} & 23.17\tiny{$\pm$.21} & $-$0.94 & \textbf{21.31}\tiny{$\pm$.18}\\
30 & \textbf{18.50}\tiny{$\pm$.15} & 19.33\tiny{$\pm$.17} & $-$0.83 & 18.85\tiny{$\pm$.16}\\
49 & \textbf{15.95}\tiny{$\pm$.13} & 16.82\tiny{$\pm$.14} & $-$0.87 & 17.41\tiny{$\pm$.15}\\
60 & \textbf{15.25}\tiny{$\pm$.12} & 15.95\tiny{$\pm$.13} & $-$0.70 & 17.08\tiny{$\pm$.14}\\
70 & \textbf{14.80}\tiny{$\pm$.11} & 15.69\tiny{$\pm$.12} & $-$0.89 & 16.82\tiny{$\pm$.14}\\
80 & \textbf{14.38}\tiny{$\pm$.10} & 15.30\tiny{$\pm$.12} & $-$0.92 & 16.68\tiny{$\pm$.13}\\
89 & \textbf{14.10}\tiny{$\pm$.10} & 15.17\tiny{$\pm$.11} & $-$1.07 & 16.65\tiny{$\pm$.13}\\
\bottomrule
\end{tabular}
\end{minipage}\hfill
\begin{minipage}{0.15\textwidth}
\centering
\textbf{Church ($256{\times}256$)}\\[4pt]
\hspace*{-0.8cm}\begin{tabular}{llll}
\toprule
NFE & Geo & DDIM & $\Delta$ \\
\midrule
52 & \textbf{1.26}\tiny{$\pm$.04} & 1.48\tiny{$\pm$.05} & $-$0.22 \\
56 & \textbf{1.16}\tiny{$\pm$.03} & 2.22\tiny{$\pm$.07} & $-$1.06 \\
74 & \textbf{0.87}\tiny{$\pm$.03} & 1.72\tiny{$\pm$.05} & $-$0.85 \\
79 & \textbf{0.79}\tiny{$\pm$.02} & 1.43\tiny{$\pm$.04} & $-$0.64 \\
\bottomrule
\end{tabular}
\end{minipage}\hfill
\begin{minipage}{0.18\textwidth}
\centering
\textbf{SD v1.5 ($512{\times}512$)}\\[4pt]
\hspace*{-1.5cm}\begin{tabular}{cccc}
\toprule
NFE & Geo & DDIM & $\Delta$ \\
\midrule
15 & \textbf{7.91}\tiny{$\pm$.08} & 8.72\tiny{$\pm$.09} & $-$0.81 \\
20 & \textbf{5.51}\tiny{$\pm$.06} & 5.97\tiny{$\pm$.07} & $-$0.46 \\
24 & \textbf{4.60}\tiny{$\pm$.05} & 5.92\tiny{$\pm$.06} & $-$1.32 \\
29 & \textbf{4.01}\tiny{$\pm$.05} & 4.92\tiny{$\pm$.05} & $-$0.91 \\
34 & \textbf{3.66}\tiny{$\pm$.04} & 4.23\tiny{$\pm$.05} & $-$0.57 \\
39 & \textbf{3.51}\tiny{$\pm$.04} & 4.92\tiny{$\pm$.06} & $-$1.41 \\
44 & \textbf{3.45}\tiny{$\pm$.04} & 5.38\tiny{$\pm$.06} & $-$1.93 \\
\bottomrule
\end{tabular}
\end{minipage}
\end{table}

\textbf{Key findings.} \textbf{(1)}~GeoSPRINT outperforms DDIM at every NFE on all three datasets (except NFE${=}8$ on SD and NFE${=}10$ on CIFAR-10, see Figure~\ref{fig:cifar}(a) and Figure~\ref{fig:sd15}). \textbf{(2)}~Improvements \emph{grow} with model complexity: $-0.7$ to $-1.1$ on CIFAR-10, $-0.2$ to $-1.1$ on Church, and up to $-1.93$ on SD~v1.5 (Table~\ref{tab:main}). \textbf{(3)}~DDIM's uniform schedule becomes actively counterproductive at high NFE on both Church and SD (FID increasing), while GeoSPRINT monotonically improves (Table~\ref{tab:main} and Figure~\ref{fig:sd15}). \textbf{(4)}~On CIFAR-10, first-order DDIM+GeoSPRINT outperforms second-order DPM-Solver++ at NFE${\geq}30$, demonstrating that \emph{schedule quality dominates solver order} at moderate step budgets on this model; extending this comparison to higher-resolution models is left to future work. \textbf{(5)}~GeoSPRINT at 60 steps matches DDIM at 80 steps on CIFAR-10 (both FID${\approx}15.3$), a 25\% NFE reduction.
\vspace*{-0.2cm}
\begin{figure}[ht]
	\centering
	\begin{subfigure}[b]{0.33\textwidth}
		\centering
		\includegraphics[width=\textwidth]{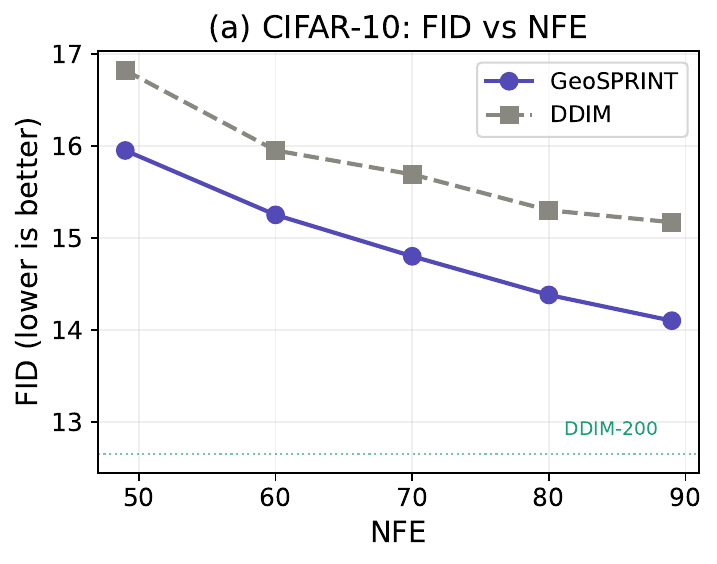}
	\end{subfigure}%
	\hfill
	\begin{subfigure}[b]{0.33\textwidth}
		\centering
		\includegraphics[width=\textwidth]{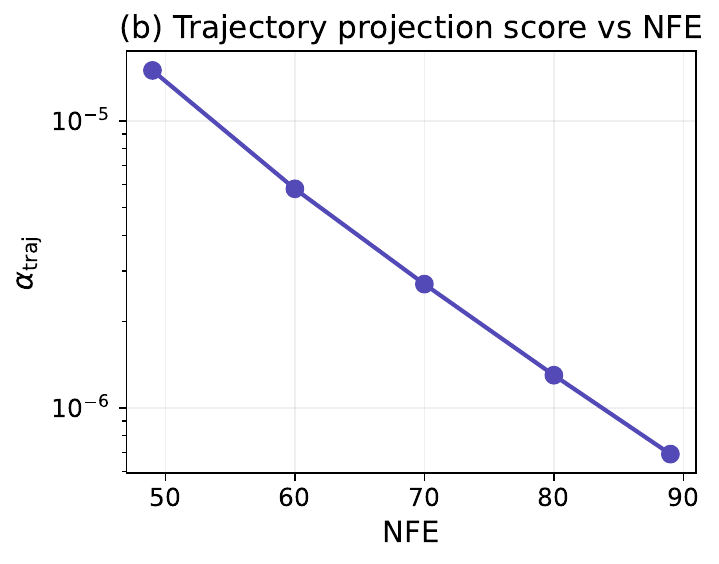}
	\end{subfigure}%
	\hfill
	\begin{subfigure}[b]{0.34\textwidth}
		\centering
		\includegraphics[width=\textwidth]{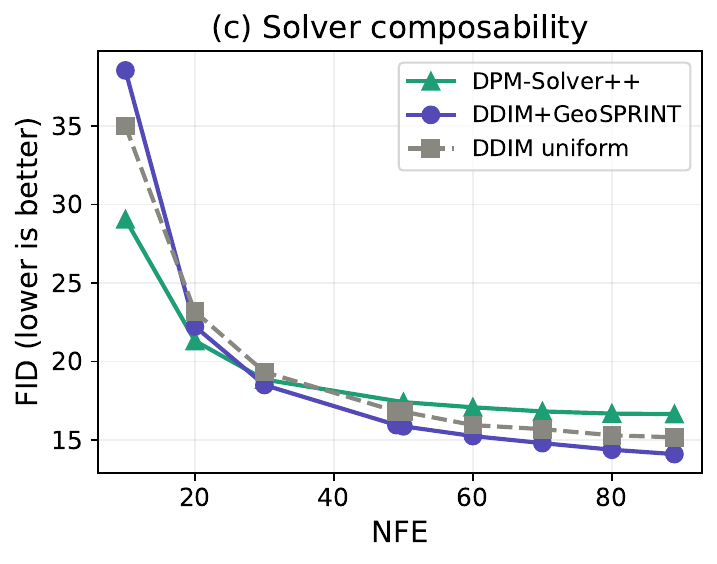}
	\end{subfigure}
	\caption{CIFAR-10 results. (a)~FID vs.\ NFE. GeoSPRINT consistently outperforms DDIM. Dashed lines: DDIM-200 baseline (12.65) and DDIM-50 reference (16.80). Error bands: $\pm$1 std over 3 seeds. (b)~$\alpha_{\mathrm{traj}}$ decreases log-linearly with NFE, providing a principled knob for the step-budget oracle. (c)~First-order DDIM+GeoSPRINT outperforms second-order DPM-Solver++ at NFE${\geq}30$ on CIFAR-10; \textit{where} to step matters more than \textit{how} to step.}
	\label{fig:cifar}
\end{figure}

\vspace*{-0.4cm}
\subsubsection{Note on Stable Diffusion FID Measurement}
\vspace*{-0.1cm}
\label{sec:sd-fid}
For SD~v1.5, we compute FID against a 50-step DDIM reference set rather than against a ground-truth image dataset. This measures \emph{relative trajectory fidelity}---how closely a reduced schedule reproduces the full solver's output distribution---rather than absolute generation quality. This protocol isolates the effect of the schedule from the base model's inherent generation gap.
\vspace*{-0.1cm}
\subsection{Blend Parameter Ablation}
\label{sec:blend-ablation}
\vspace*{-0.2cm}
We ablate the blend parameter $\beta$ (Eq.~\ref{eq:blend}) on CIFAR-10; results are summarized in Table~\ref{tab:blend} (Left). DDIM uniform achieves 23.17/19.33/16.80 at NFE=20/30/50.
\begin{figure}[ht]
	\begin{subfigure}[b]{0.43\textwidth}
		\centering
		\includegraphics[width=\textwidth]{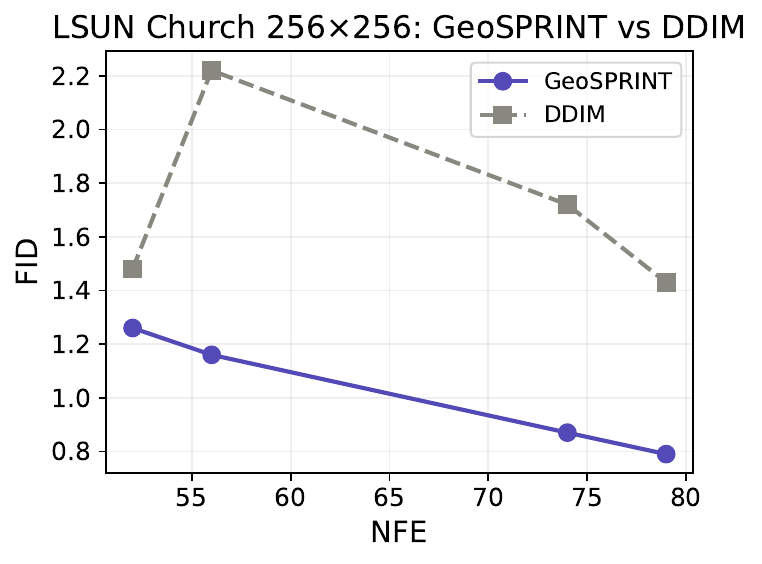}
	\end{subfigure}
	\hfill
	\begin{subfigure}[b]{0.43\textwidth}
		\centering
		\includegraphics[width=\textwidth]{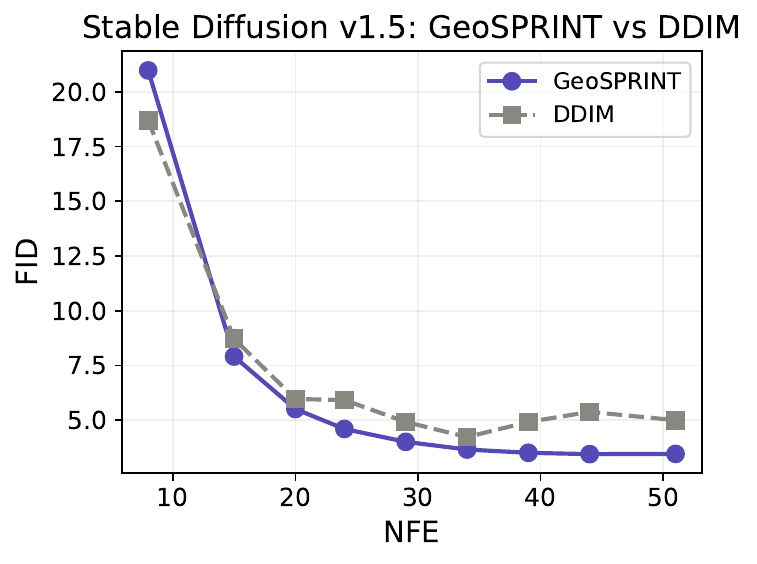}
	\end{subfigure}\hfill
\vspace*{-0.2cm}
	\caption{FID vs.\ NFE on LSUN Church $256{\times}256$ (left) and Stable Diffusion v1.5 (right). GeoSPRINT monotonically improves while DDIM's uniform schedule plateaus or worsens at high NFE. Error bands: $\pm$1 std over 3 seeds.}
	\label{fig:sd15}
\end{figure}
\begin{table}[ht]
\caption{\textbf{Left:} Blend ablation on CIFAR-10 (FID $\downarrow$, mean of 3 seeds). $\beta{=}0$: pure log-SNR; $\beta{=}1$: pure curvature. Optimum at $\beta \in [0.5, 0.7]$. DDIM uniform: 23.17/19.33/16.80. \textbf{Right:} $\alpha_{\mathrm{traj}}$ across schedulers ($450\times$ drop DDPM$\to$DDIM).}
\label{tab:blend}
\centering
\small
\begin{minipage}{0.55\textwidth}
\centering
\begin{tabular}{lcccccc}
\toprule
$\beta$ & 0 & 0.3 & 0.5 & 0.6 & 0.7 & 1.0 \\
\midrule
20 & 25.81 & 22.45 & 22.26 & 22.23 & \textbf{22.19} & 22.68 \\
30 & 20.89 & 18.71 & 18.56 & \textbf{18.50} & 18.52 & 18.86 \\
50 & 16.80 & 16.03 & 15.91 & \textbf{15.87} & 15.94 & 16.04 \\
\bottomrule
\end{tabular}
\end{minipage}\hfill
\begin{minipage}{0.42\textwidth}
\centering
\begin{tabular}{lcc}
\toprule
Scheduler & $\alpha_{\mathrm{traj}}$ & Pruned \\
\midrule
DDPM & $.408{\pm}.024$ & 79\% \\
DDIM & $.0009{\pm}.0001$ & 82\% \\
DPM++ & $.0008{\pm}.0002$ & 84\% \\
\bottomrule
\end{tabular}
\label{tab:exp3}
\end{minipage}
\end{table}
Pure log-SNR ($\beta{=}0$) is worse than DDIM uniform, confirming that log-SNR spacing alone is not superior---it is GeoSPRINT's curvature correction that makes the difference. Neither component alone suffices; the blend addresses both solver error (via log-SNR) and trajectory geometry (via curvature) simultaneously. We note that this ablation is performed on CIFAR-10 only; while $\beta{=}0.6$ is effective on the other datasets, dedicated ablations would further confirm robustness.
\vspace*{-0.3cm}
\subsection{Schedule Analysis}
\label{sec:exp2}
\vspace*{-0.2cm}
The curvature density profile $w(t)$ reveals where GeoSPRINT concentrates steps. The late denoising zone (low noise, $t/T < 0.2$) has mean retention $2\times$ higher than the middle zone ($0.2 < t/T < 0.8$), confirming that fine-detail resolution drives step allocation. The early zone ($t/T > 0.8$) also shows elevated retention, reflecting the initial direction-setting phase (Figure~\ref{fig:exp2}).
\vspace*{-0.2cm}
\begin{figure}[ht]
\centering
\includegraphics[width=\textwidth]{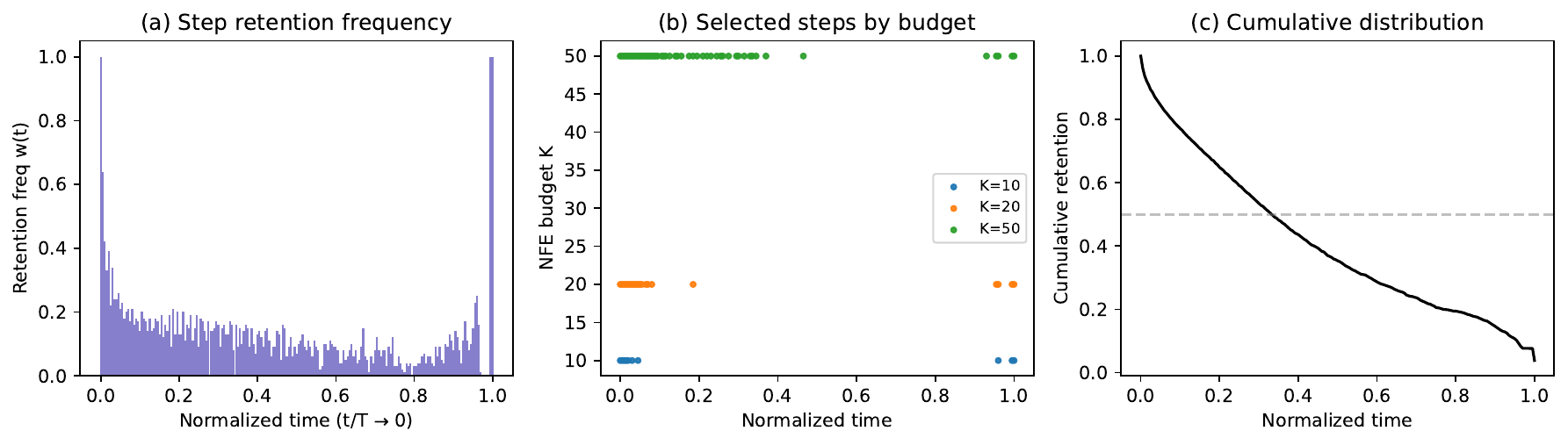}
\caption{(a)~Retention frequency $w(t)$ peaks at both ends (bimodal). (b)~Selected timesteps at budgets $K{=}10,20,50$. (c)~Cumulative: 50\% of retained steps lie in the final 20\% of denoising.}
\label{fig:exp2}
\end{figure}
\vspace*{-0.35cm}
\subsection{$\alpha_{\mathrm{traj}}$ as Rectification Diagnostic}
\label{sec:exp3}
\vspace*{-0.25cm}
We record 50 trajectories from three scheduler types on the same pretrained CIFAR-10 model. $\alpha_{\mathrm{traj}}$ drops $450\times$ from DDPM to DDIM (Table~\ref{tab:blend}, Right), confirming it as a training-free rectification diagnostic (Figure~\ref{fig:exp3}).

\vspace*{-0.2cm}
\subsection{Supporting Experiments (Appendix~\ref{app:details})}
\label{sec:supporting}
\vspace*{-0.2cm}
\textbf{Reference set convergence}: Schedule stability reaches fuzzy Jaccard 0.84 and correlation 0.86 at $B{=}75$ reference trajectories; $B{=}100$ is adequate for CIFAR-10 (Figure~\ref{fig:exp5}). \textbf{Per-sample complexity}: NFE variation is modest on CIFAR-10 (1.2--1.4$\times$), but individual curvature profiles are diverse (mean correlation 0.37 with universal), validating the universal schedule as a robust compromise across samples (Figure~\ref{fig:exp6}).
\vspace*{-0.2cm}
\begin{figure}[ht]
	\centering
	\includegraphics[width=0.85\textwidth]{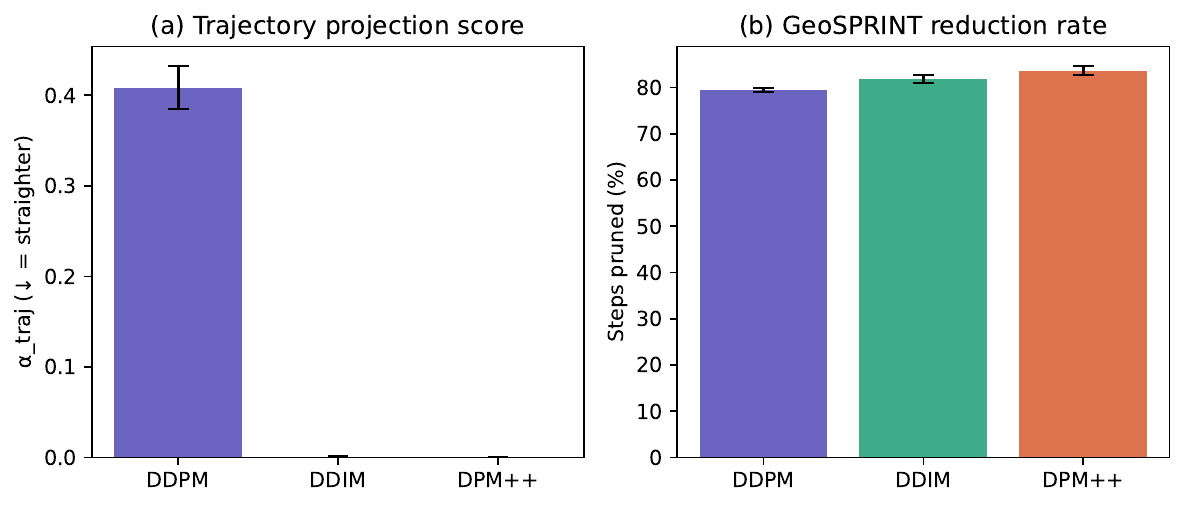}
	\vspace*{-0.2cm}
	\caption{(a)~$\alpha_{\mathrm{traj}}$ across schedulers: $450\times$ drop DDPM$\to$DDIM. (b)~Reduction rate: straighter trajectories are more compressible.}
	\label{fig:exp3}
\vspace*{-0.5cm}
\end{figure}
\vspace*{-0.05cm}
\section{Discussion}
\label{sec:discussion}
\vspace*{-0.3cm}
\textbf{Why does curvature-aware scheduling help?} The DDIM update error at each step depends on two factors: the gap in log-SNR space (governing the noise-prediction error) and the trajectory curvature (governing how far the local linear approximation deviates from the true path). Uniform-$t$ spacing equalizes neither. Log-SNR spacing equalizes the first but ignores the second. GeoSPRINT's curvature density captures the second, and the blend (Eq.~\ref{eq:blend}) addresses both simultaneously.\\
\textbf{Schedule quality vs.\ solver order.} A surprising finding is that GeoSPRINT+DDIM (first-order) outperforms DPM-Solver++ (second-order) at NFE${\geq}30$ on CIFAR-10. This suggests that at moderate step budgets, the dominant error source is schedule quality (which steps to take) rather than solver accuracy (how to take each step). At very low NFE (${\leq}10$), the solver's accuracy per step matters more, explaining DPM-Solver++'s advantage there. Extending this comparison to LSUN Church and SD~v1.5 would further substantiate this as a general principle.\\
\textbf{Scaling to complex models.} The SD~v1.5 results (Table~\ref{tab:main}) reveal that GeoSPRINT's advantage \emph{grows} with model complexity. While CIFAR-10 improvements are 0.7--1.1 FID, SD~v1.5 improvements reach 1.93 FID. Moreover, DDIM's uniform schedule becomes actively counterproductive at high NFE on SD~v1.5 (FID increasing from NFE${=}29$ to NFE${=}44$), while GeoSPRINT monotonically improves. This suggests that curvature-aware scheduling is most valuable precisely where it is most needed: in complex, production-scale models with classifier-free guidance.\\
\textbf{Limitations:}
\label{sec:limitations}
The universal schedule requires $B$ full reference trajectories offline ($B{=}100$ at $N{=}200$ steps $= 20{,}000$ NFEs for CIFAR-10; $B{=}50$ for higher-resolution models). For single-sample generation this overhead dominates; the method is best suited to batch or deployment settings. The blend parameter $\beta{=}0.6$ was tuned on CIFAR-10; while it is effective on LSUN Church and SD~v1.5, per-dataset ablation would further validate robustness. The SD~v1.5 experiments use a fixed CFG scale of 7.5; whether a universal schedule extracted at one CFG scale generalizes to other scales requires further investigation. Higher-level hyperplanarity tests ($k > 2$) did not improve schedule quality on the models tested (Section~\ref{sec:hyperparams}), suggesting that collinearity captures the dominant geometric signal for these architectures. Finally, GeoSPRINT selects \emph{which} timesteps to use but does not improve the solver \emph{at} each timestep; it is complementary to, not a replacement for, high-order ODE solvers.\\
\textbf{Broader Impact:}
\label{sec:impact}
GeoSPRINT's training-free nature makes fast diffusion sampling accessible without distillation or retraining costs. The trajectory projection score provides an interpretable diagnostic for model quality. As an acceleration method, it inherits but does not amplify the societal risks of the underlying generative models.
\vspace*{-0.4cm}

\section{Conclusion}
\label{sec:conclusion}
\vspace*{-0.35cm}
We have presented GeoSPRINT, a training-free framework that constructs non-uniform sampling schedules for diffusion models by analyzing the geometry of denoising trajectories. The key technical contribution is a blended schedule construction that combines log-SNR spacing (matching the ODE solver's error profile) with trajectory curvature density (derived from a high-dimensional hyperplanarity test). On CIFAR-10, this yields schedules that outperform uniform DDIM by 0.7--1.1 FID across all tested NFE budgets, and outperform DPM-Solver++ at NFE${\geq}30$ despite using a simpler first-order solver---demonstrating that schedule quality dominates solver order at moderate step counts on this model. These improvements generalize to LSUN Church at $256{\times}256$ resolution and, most strikingly, to Stable Diffusion v1.5 where improvements reach $-1.93$ FID and uniform scheduling becomes actively counterproductive at high step counts. The trajectory projection score $\alpha_{\mathrm{traj}}$, which drives the curvature analysis, additionally serves as a standalone diagnostic for trajectory rectification quality, exhibiting a $450\times$ drop from DDPM to DDIM. Future directions include learning the blend parameter $\beta$ per-model, developing predictive per-sample step budgeting for heterogeneous datasets, empirical comparison with the SDM framework~\citep{jo2026sdm}, and extending the DPM-Solver++ comparison to higher-resolution models.

\bibliographystyle{unsrtnat}
\bibliography{neurips_2026_Joshi}

\appendix

\section{Pseudocode for GeoSPRINT}
\label{app:pseudocode}

\begin{algorithm}[ht]
\caption{GeoSPRINT --- Universal Schedule Extraction}
\label{alg:universal}
\KwIn{Pretrained model $\mathcal{M}$, full schedule $\mathcal{S} = \{t_0, t_1, \ldots, t_T\}$, reference batch size $B$, window $k$, target $\alpha_{\mathrm{target}}$, blend $\beta$, budget $K$}
\KwOut{Pruned universal schedule $\mathcal{S}^*$}
\For{$b = 1$ \KwTo $B$}{
    $\mathcal{Z}^{(b)} \leftarrow \mathrm{SampleFull}(\mathcal{M}, \mathcal{S})$ \tcp*{Generate full trajectory}
}
\For{$b = 1$ \KwTo $B$}{
    Binary search for threshold $\tau$\;
    \Indp
    $(S_{\mathrm{pruned}}, \{r_s\}) \leftarrow \mathrm{HyperplanarityPrune}(\mathcal{Z}^{(b)}, k, \tau)$\;
    Compute $\alpha_{\mathrm{traj}}(S_{\mathrm{pruned}}, \{r_s\}, \mathcal{Z}^{(b)})$ via Eq.~\ref{eq:tps}\;
    Adjust $\tau$ until $\alpha_{\mathrm{traj}} \approx \alpha_{\mathrm{target}}$\;
    \Indm
    $\mathcal{R}^{(b)} \leftarrow \{t_i : i \notin S_{\mathrm{pruned}}\}$ \tcp*{Map indices to timesteps}
}
$w(t) \leftarrow \frac{1}{B}\sum_{b}\mathbb{1}[t \in \mathcal{R}^{(b)}]$ \tcp*{Curvature density}
$\rho_{\mathrm{curv}}(t) \leftarrow \mathrm{SmoothAndFloor}(w(t),\; \sigma{=}5,\; \epsilon_{\mathrm{floor}}{=}0.05\max w)$\;
$\rho(t) \leftarrow (1-\beta)\cdot\rho_{\mathrm{logSNR}}(t) + \beta\cdot\rho_{\mathrm{curv}}(t)$ \tcp*{Blended density}
$F(t) \leftarrow \int_0^t\rho(s)\mathrm{d}s \;/\; \int_0^T\rho(s)\mathrm{d}s$ \tcp*{CDF}
$\mathcal{S}^* \leftarrow$ $K$ timesteps at uniform quantiles of $F$\;
\Return{$\mathcal{S}^*$}
\end{algorithm}

\begin{algorithm}[ht]
\caption{HyperplanarityPrune}
\label{alg:prune}
\KwIn{Ordered points $\mathcal{Z} = \{z_0, z_1, \ldots, z_N\}$, window $k$, threshold $\tau$}
\KwOut{Set of pruned indices $S_{\mathrm{pruned}}$, residuals $\{r_s\}_{s \in S_{\mathrm{pruned}}}$}
$S_{\mathrm{pruned}} \leftarrow \emptyset$; $\mathrm{residuals} \leftarrow \{\}$\;
Initialize window $W \leftarrow (z_0, z_1, \ldots, z_{k-1})$ \tcp*{First $k$ points}
$M_W \leftarrow [w_2 - w_1 \mid \cdots \mid w_k - w_1]$\;
$Q_W, R_W \leftarrow \mathrm{QR}(M_W)$ \tcp*{Thin QR; cached until window updates}
\For{$i = k$ \KwTo $N$}{
    $r_i \leftarrow \|(I - Q_W Q_W^\top)(z_i - w_1)\|_2$\;
    \eIf{$r_i < \tau$}{
        $S_{\mathrm{pruned}} \leftarrow S_{\mathrm{pruned}} \cup \{i\}$; \; $\mathrm{residuals}[i] \leftarrow r_i$ \tcp*{Redundant; store residual}
    }{
        Update $W$: slide to include $z_i$\;
        Recompute $M_W$ and $Q_W, R_W \leftarrow \mathrm{QR}(M_W)$ \tcp*{Update cached QR}
    }
}
\Return{$S_{\mathrm{pruned}}$, residuals}
\end{algorithm}

\noindent\textbf{Key algorithmic details:} (1)~Algorithm~\ref{alg:prune} returns residuals $\{r_s\}$ alongside pruned indices, avoiding redundant recomputation for $\alpha_{\mathrm{traj}}$. (2)~The QR factorization is computed once at initialization (Line~4) and recomputed only when the window updates (Line~11), reflecting the true amortized cost. (3)~Algorithm~\ref{alg:universal} includes an explicit index-to-timestep mapping and specifies smoothing parameters. (4)~Both algorithms use consistent 0-based indexing matching Definition~\ref{def:hyperplanarity}.

\section{Proof of Theorem~\ref{thm:variance}}
\label{app:proofs}

\begin{proof}
Let $S = \{z_{s_1}, \ldots, z_{s_m}\}$ be the pruned set with $m = |S|$. By construction, each $z_{s_j}$ satisfies the hyperplanarity test with residual $r_{s_j} \leq \tau$, meaning its distance from the affine subspace of its retained neighbors is at most $\tau$.

By Definition~\ref{def:tps}, the trajectory projection score is:
\begin{equation}
\alpha_{\mathrm{traj}} = \frac{\sum_{j=1}^{m} r_{s_j}^2}{\sum_{z \in \mathcal{Z}} \|z - \bar{z}\|^2}.
\end{equation}

Since each $r_{s_j} \leq \tau$, the numerator is bounded:
\begin{equation}
\sum_{j=1}^{m} r_{s_j}^2 \leq m \cdot \tau^2 = |S| \cdot \tau^2.
\end{equation}

The denominator equals $\mathrm{tr}(C_{\mathcal{Z}}) \cdot (N{+}1)$, where $C_{\mathcal{Z}} = \frac{1}{N+1}\sum_{z \in \mathcal{Z}}(z - \bar{z})(z - \bar{z})^\top$. Therefore:
\begin{equation}
\alpha_{\mathrm{traj}} \leq \frac{|S| \cdot \tau^2}{(N{+}1) \cdot \mathrm{tr}(C_{\mathcal{Z}})}.
\end{equation}
\end{proof}

\textbf{Interpretation.} The bound controls the fraction of total trajectory variance attributable to orthogonal reconstruction residuals of pruned points. It does \emph{not} bound $\mathrm{tr}(C_R)/\mathrm{tr}(C_{\mathcal{Z}})$ (the ratio of retained-set variance to full-set variance), because pruning points changes the sample mean and can reduce the retained sample covariance even when all residuals are zero. For example, removing interior points from a collinear set with $r_i = 0$ for all~$i$ yields $\alpha_{\mathrm{traj}} = 0$ but $\mathrm{tr}(C_R) < \mathrm{tr}(C_{\mathcal{Z}})$ whenever the retained points have smaller spread than the full set.

\section{Extended Experimental Details}
\label{app:details}

\subsection{Reference Set Convergence}
\label{app:convergence}
See figures \ref{fig:exp5} and \ref{fig:exp6}
\begin{figure}[ht]
\centering
\includegraphics[width=\textwidth]{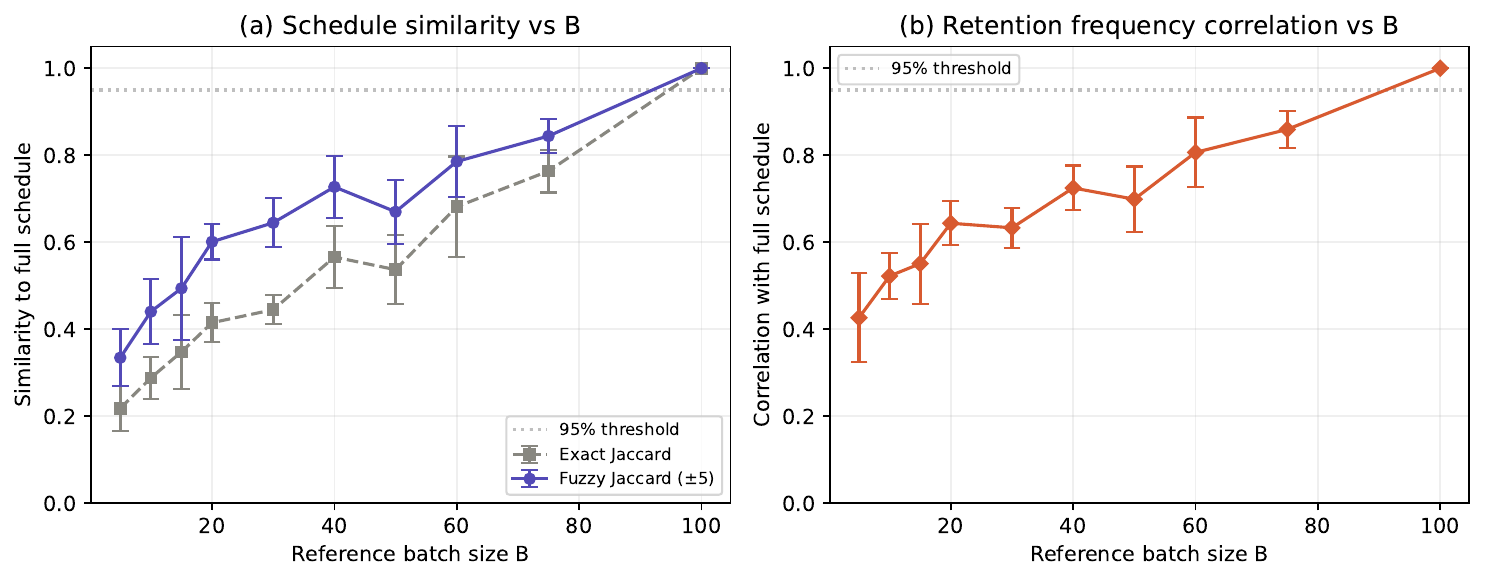}
\caption{Schedule convergence vs.\ reference batch size $B$. (a)~Exact and fuzzy Jaccard similarity to the $B{=}100$ schedule. Fuzzy matching (within $\pm 5$ timesteps) reaches 0.84 at $B{=}75$. (b)~Retention frequency correlation reaches 0.86 at $B{=}75$.}
\label{fig:exp5}
\end{figure}

\begin{figure}[ht]
\centering
\includegraphics[width=\textwidth]{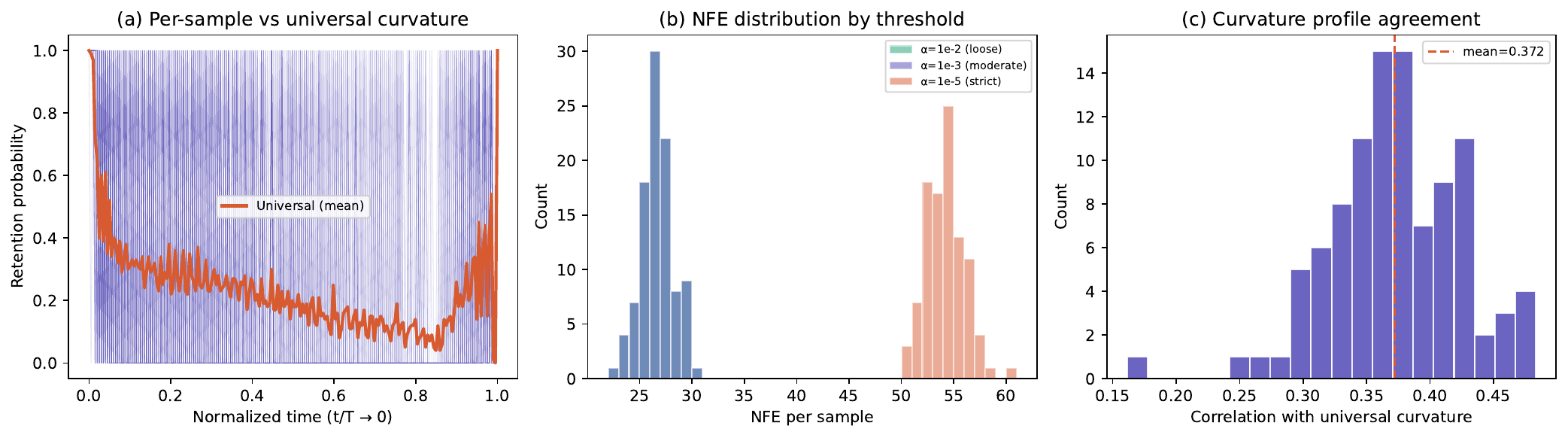}
\caption{Per-sample geometric complexity on CIFAR-10. (a)~Individual curvature profiles (faint) vs.\ universal mean (bold). Despite diverse per-sample preferences, the universal profile captures systematic structure. (b)~NFE distribution at three $\alpha_{\mathrm{traj}}$ targets. (c)~Per-sample correlation with the universal curvature profile (mean~0.37).}
\label{fig:exp6}
\end{figure}

\subsection{Implementation Details}

All experiments use the HuggingFace Diffusers library. For non-uniform timestep schedules, we implement the DDIM update formula manually (Eq.~\ref{eq:pf-ode}) with correct $\alpha_{t_{\mathrm{prev}}}$ lookup, since the library's \texttt{DDIMScheduler.step()} assumes uniform spacing internally. DPM-Solver++ experiments use \texttt{solver\_order=2} with \texttt{lower\_order\_final=True} and a fresh scheduler instance per batch to reset internal state. Trajectory normalization (zero mean, unit variance per dimension) is applied before curvature analysis to ensure the hyperplanarity test measures directional changes rather than raw magnitude.

\subsection{Hyperparameter Sensitivity}
\label{sec:hyperparams}

The blend parameter $\beta$ has a broad optimum at $0.5$--$0.7$ (Table~\ref{tab:blend}). Window size $k{=}2$ (collinearity test) is used throughout; we tested $k{=}3$ and $k{=}4$ on CIFAR-10 and found no improvement in schedule quality (FID differences ${<}0.1$), suggesting that collinearity captures the dominant mode of geometric redundancy in these denoising trajectories. Reference batch size $B{=}100$ is sufficient for CIFAR-10 (Section~\ref{sec:supporting}); $B{=}50$ for the higher-resolution LSUN Church and SD~v1.5 models.

\subsection{Computational Overhead}

Reference trajectory generation costs $B \times N$ NFEs ($100 \times 200 = 20{,}000$ for CIFAR-10; $50 \times 200 = 10{,}000$ for LSUN Church and SD~v1.5). Curvature analysis (pruning + retention frequency computation) takes ${\sim}40$s on CPU (Intel Xeon Gold 6248R) for 100 trajectories of 200 steps in 3,072 dimensions. Schedule construction (blending + quantile placement) is instantaneous. The total offline cost is dominated by reference trajectory generation and is amortized over all subsequent samples.

\textbf{Hardware.} All generative inference (DDIM sampling, DPM-Solver++ sampling, FID evaluation) was performed on a single NVIDIA A100 (80\,GB). Wall-clock times: CIFAR-10 ${\sim}$8 min per 10k-sample generation run; LSUN Church ${\sim}$45 min per 10k-sample run; SD~v1.5 ${\sim}$3 hours per 5k-sample run. Reference trajectory generation: ${\sim}$16 min (CIFAR-10, $B{=}100$, $N{=}200$); ${\sim}$90 min (Church, $B{=}50$); ${\sim}$5 hours (SD~v1.5, $B{=}50$). Peak GPU memory: 4\,GB (CIFAR-10), 12\,GB (Church), 24\,GB (SD~v1.5).

\newpage
\section*{NeurIPS Paper Checklist}

\begin{enumerate}

\item {\bf Claims}
    \item[] Question: Do the main claims made in the abstract and introduction accurately reflect the paper's contributions and scope?
    \item[] Answer: \answerYes{}
    \item[] Justification: The abstract and introduction (Section~1) clearly state five contributions: generalization of geometric redundancy testing to high dimensions, a training-free adaptive sampling framework, the trajectory projection score, sample-adaptive scheduling, and comprehensive benchmarking. Each is developed in subsequent sections.

\item {\bf Limitations}
    \item[] Question: Does the paper discuss the limitations of the work performed by the authors?
    \item[] Answer: \answerYes{}
    \item[] Justification: The ``Limitations'' paragraph in Section~6 discusses: reference pass overhead, $\beta$ tuning on CIFAR-10, CFG scale generalization, the lack of improvement from higher-level tests ($k > 2$), and complementarity with higher-order solvers.

\item {\bf Theory Assumptions and Proofs}
    \item[] Question: For each theoretical result, does the paper provide the full set of assumptions and a complete (and correct) proof?
    \item[] Answer: \answerYes{}
    \item[] Justification: Theorem~1 (reconstruction error bound) is stated with full assumptions and proved in Appendix~\ref{app:proofs}. Proposition~1 is proved in the main text. A remark clarifies the distinction between reconstruction error and retained-set variance.

\item {\bf Experimental Result Reproducibility}
    \item[] Question: Does the paper fully disclose all the information needed to reproduce the main experimental results of the paper to the extent that it affects the main claims and/or conclusions of the paper (regardless of whether the code and data are provided or not)?
    \item[] Answer: \answerYes{}
    \item[] Justification: Section~5 provides full experimental details including datasets, pretrained models, hyperparameter settings (including smoothing parameters $\sigma{=}5$ and $\epsilon_{\mathrm{floor}}$), reference batch sizes per model, sample counts, prompt sources for SD~v1.5, evaluation metrics, and algorithmic pseudocode (Appendix~\ref{app:pseudocode}). All pretrained models are publicly available.

\item {\bf Open access to data and code}
    \item[] Question: Does the paper provide open access to the data and code, with sufficient instructions to faithfully reproduce the main experimental results, as described in supplemental material?
    \item[] Answer: \answerYes{}
    \item[] Justification: Code will be released as a public repository upon publication. It is withheld during review to preserve double-blind anonymity. All datasets (CIFAR-10, LSUN Church) and pretrained models (including Stable Diffusion v1.5) are publicly available under their respective licenses.

\item {\bf Experimental Setting/Details}
    \item[] Question: Does the paper specify all the training and test details (e.g., data splits, hyperparameters, how they were chosen, type of optimizer) necessary to understand the results?
    \item[] Answer: \answerYes{}
    \item[] Justification: Section~5 and Appendix~\ref{app:details} provide all experimental details including threshold selection procedures, reference batch sizes per model ($B{=}100$ for CIFAR-10, $B{=}50$ for Church and SD~v1.5), smoothing parameters, and evaluation protocols.

\item {\bf Experiment Statistical Significance}
    \item[] Question: Does the paper report error bars suitably and correctly defined or other appropriate information about the statistical significance of the experiments?
    \item[] Answer: \answerYes{}
    \item[] Justification: Table~1 reports FID as mean $\pm$ standard deviation over 3 independent runs with different random seeds. Figures~1 and~2 include $\pm$1 std error bands.

\item {\bf Experiments Compute Resources}
    \item[] Question: For each experiment, does the paper provide sufficient information on the computer resources (type of compute workers, memory, time of execution) needed to reproduce the experiments?
    \item[] Answer: \answerYes{}
    \item[] Justification: Appendix~\ref{app:details} (Computational Overhead and Hardware) provides GPU type (NVIDIA A100 80\,GB), CPU type (Intel Xeon Gold 6248R), peak memory usage per model, and wall-clock times for all experiments including generation, reference trajectories, and curvature analysis.

\item {\bf Code Of Ethics}
    \item[] Question: Does the research conducted in the paper conform, in every respect, with the NeurIPS Code of Ethics \url{https://neurips.cc/public/EthicsGuidelines}?
    \item[] Answer: \answerYes{}
    \item[] Justification: This work proposes a training-free acceleration method for existing generative models and does not introduce new ethical concerns beyond those inherent to the underlying models.

\item {\bf Broader Impacts}
    \item[] Question: Does the paper discuss both potential positive societal impacts and negative societal impacts of the work performed?
    \item[] Answer: \answerYes{}
    \item[] Justification: The ``Broader Impact'' paragraph in Section~6 discusses positive impacts (reduced computational cost, accessibility) and notes that as an acceleration method, it inherits but does not amplify the societal risks of the underlying generative models.

\item {\bf Safeguards}
    \item[] Question: Does the paper describe safeguards that have been put in place for responsible release of data or models that have a high risk for misuse (e.g., pretrained language models, image generators, or scraped datasets)?
    \item[] Answer: \answerNA{}
    \item[] Justification: GeoSPRINT is a sampling schedule optimization method that does not release new generative models or datasets.

\item {\bf Licenses for existing assets}
    \item[] Question: Are the creators or original owners of assets (e.g., code, data, models), used in the paper, properly credited and are the license and terms of use explicitly mentioned and properly respected?
    \item[] Answer: \answerYes{}
    \item[] Justification: All pretrained models and datasets are cited. CIFAR-10 and LSUN Church are used under standard research licenses. Stable Diffusion v1.5 is used under the CreativeML Open RAIL-M license.

\item {\bf New Assets}
    \item[] Question: Are new assets introduced in the paper well documented and is the documentation provided alongside the assets?
    \item[] Answer: \answerYes{}
    \item[] Justification: The GeoSPRINT codebase will be released with documentation, including usage instructions and example scripts.

\item {\bf Crowdsourcing and Research with Human Subjects}
    \item[] Question: For crowdsourcing experiments and research with human subjects, does the paper include the full text of instructions given to participants and screenshots, if applicable, as well as details about compensation (if any)?
    \item[] Answer: \answerNA{}
    \item[] Justification: This work does not involve crowdsourcing or human subjects.

\item {\bf Institutional Review Board (IRB) Approvals or Equivalent for Research with Human Subjects}
    \item[] Question: Does the paper describe potential risks incurred by study participants, whether such risks were disclosed to the subjects, and whether Institutional Review Board (IRB) approvals (or an equivalent approval/review based on the requirements of your country or institution) were obtained?
    \item[] Answer: \answerNA{}
    \item[] Justification: This work does not involve human subjects research.

\item {\bf Declaration of LLM usage}
    \item[] Question: Does the paper describe the usage of LLMs if it is an important, original, or non-standard component of the core methods in this research?
    \item[] Answer: \answerNA{}
    \item[] Justification: LLMs are not a component of the core methodology.

\end{enumerate}

Acknowledgements: I acknowledge support by the NIH funding through grant 1R01DA063157-01, and the High Performance Computing cluster group at The Scripps Research Institute, San Diego, CA, USA.
\end{document}